\documentclass[letterpaper, 10 pt, conference]{ieeeconf}  

\IEEEoverridecommandlockouts                              

\usepackage{graphicx}          

\usepackage{etex}        
\usepackage{amsfonts}
\usepackage{times}
\usepackage{latexsym,amsfonts,amsbsy}
\usepackage{amsmath}
\usepackage{amssymb}
\usepackage{hyperref}

\usepackage{textcomp}
\usepackage{gensymb}

\usepackage[absolute,overlay]{textpos}
\usepackage{amsbsy}
\usepackage{stmaryrd}
\usepackage{MnSymbol}

\usepackage{comment}
\usepackage{mathtools}
\usepackage{amsthm}

\usepackage{algorithm}
\usepackage{algpseudocode}

\newcommand*\Let[2]{\State #1 $\gets$ #2}
\algrenewcommand\algorithmicrequire{\textbf{Input:}}
\algrenewcommand\algorithmicensure{\textbf{Output:}}

\usepackage{listings}

\algnewcommand\algorithmicinput{\textbf{INPUT:}}
\algnewcommand\INPUT{\item[\algorithmicinput]}

\usepackage{tabularx,booktabs}
\usepackage{epstopdf}
\allowdisplaybreaks
\usepackage{mwe}    
\usepackage{caption}
\usepackage{subcaption}
\usepackage{color}
\usepackage{cite}
\usepackage{float}

\usepackage{tikz}
\usetikzlibrary{arrows,automata}
\usetikzlibrary{calc}
\usetikzlibrary{shapes.geometric}
\usetikzlibrary{positioning}

\tikzstyle{startstop} = [rectangle, rounded corners, minimum width=3cm, minimum height=1cm,text centered, draw=black]
\tikzstyle{process} = [rectangle, rounded corners,minimum width=3cm, minimum height=.8cm,, text centered, draw=black]
\tikzstyle{decision} = [diamond, minimum width=2cm, minimum height=1cm, text centered, draw=black]
\tikzstyle{arrow} = [thick,->,>=stealth]
\tikzstyle{blank} = [node distance=1cm]
\tikzstyle{block} = [rectangle, draw, fill=blue!20, text centered, rounded corners, minimum height=1em]
\tikzstyle{line} = [draw, -latex']
\tikzstyle{cloud} = [draw, ellipse,fill=red!20, node distance=4cm,minimum height=1em]
\title{\LARGE \bf
Reactive Integrated Mission and Motion Planning}

\author{Alireza~Partovi, Rafael~Rodrigues~da~Silva$^1$ and Hai~Lin
	\thanks{The partial support of the National Science Foundation (Grant No. CNS-1446288, ECCS-1253488, IIS-1724070) and of the Army Research Laboratory (Grant No. W911NF- 17-1-0072) is gratefully acknowledged.}
	\thanks{The authors are with the Department of Electrical Engineering, University of Notre Dame, Notre Dame,
		IN, 46556 USA. {\tt\small apartovi@nd.edu, rrodri17@nd.edu, hlin1@nd.edu}}
    \thanks{$^1$ The second author would like to appreciate the scholarship support by CAPES/BR, BEX 13242/13-0}}

\newtheorem{definition}{Definition}
\newtheorem{problem}{Problem}

\newtheorem{thm}{Theorem}

\newtheorem{prop}{Proposition}
\newtheorem{lemma}{Lemma}

\begin{document}
	
\maketitle
\thispagestyle{empty}
\pagestyle{empty}

\begin{abstract} 
   Correct-by-construction manipulation planning in a dynamic environment, where other agents can manipulate objects in the workspace, is a challenging problem. The tight coupling of actions and motions between agents and complexity of mission specifications makes the problem computationally intractable. 
   This paper presents a reactive integrated mission and motion planning for mobile-robot manipulator systems operating in a partially known environment. We introduce a multi-layered synergistic framework that receives high-level mission specifications expressed in linear temporal logic and generates dynamically-feasible and collision-free motion trajectories to achieve it. In the high-level layer, a mission planner constructs a  symbolic two-player game between the robots and their environment to synthesis a strategy that adapts to changes in the workspace imposed by other robots. A bilateral synergistic layer is developed to map the designed mission plan to an integrated task and motion planner, constructing a set of robot tasks to move the objects according to the mission strategy.   In the low-level planning stage,  verifiable motion controllers are designed that can be incrementally composed to guarantee a safe motion planning for each high-level induced task. 
The proposed framework is illustrated with a multi-robot warehouse example with the mission of moving objects to various locations.  
\end{abstract}

\section{Introduction}
In the recent years, there has been an increased interest in using autonomous manipulator robots in factory automation  \cite{wurman2008coordinating}. 
 In a warehouse, for instance,  manipulator mobile robots are tasked to perform sequences of pick-up/drop-off objects and meanwhile are expected to have a collision-free path planning in the operating workspace. Exhibiting such a complex behavior requires integration of a high-level mission planner to synthesize long-term strategies, and a motion planner to generate feasible motion trajectories concerning the robot's dynamic model and workspace structure.  
Seamlessly combining these layers, however, leads to challenging class of planning problems  \cite{cambon2009hybrid,he2015towards}. 

A key challenge is to adapt the task and motion planning with a change of the workspace structure. 
Transitionally, manipulation planning is performed in a highly structured environment that allows using pre-computed motion trajectories \cite{belta2007symbolic}. 
In the multi-robot manipulator systems, however, each robot may change the workspace structure by carrying the objects to different locations which may create infeasible motion trajectories and fails the pre-computed task plans.
The robots, therefore, at the high-level planning are required to fulfill the mission objective by reactively adapt the strategy concerning other robots behavior, and at the low-level motion planning, to ensure a collision-free movement in a dynamic workspace.

Another important challenge is the complexity of defining the mission objectives. In the multi-robot manipulation planning with a complex mission, including a detailed description of each robot's task in the mission specification can be tedious \cite{he2015towards}.
This demands a correct-by-construction design that the user only expresses the mission requirement over the objects' location of interest, and the robots synthesize the necessary detailed tasks and movements to fulfill it.
This paper aims to develop a  correct-by-construction integrated mission and motion planning framework for robotic manipulator system with mission objective expressed in linear temporal logic.

Integrated mission and motion planning in a dynamic environment are demonstrated in other mobile robot applications. Authors in \cite{kress2011correct,alonso2017reactive}, 
 develops a framework for sensor-based temporal logic motion planning where the reactive mission planner responds to changes in the environment.  However, they are limited to reach-avoid specifications.

Various works leveraged AI approaches for high-level task planning \cite{cambon2009hybrid,kaelbling2011hierarchical,kaelbling2013integrated,gharbi2015combining}. AsYmov \cite{cambon2009hybrid} combines Probabilistic Roadmaps Methods (PRM) with a heuristic-search task planner based on metric-FF. However, these heuristic search algorithms are limited, because it ignores the geometric details in the symbolic layer. Thus, Hierarchical Task Networks (HTN) is extended in several related approaches \cite{kaelbling2011hierarchical,kaelbling2013integrated,gharbi2015combining} with shared literals to control backtracking between task and motion layer. For example, Hierarchical Task in the Now (HTN) \cite{kaelbling2011hierarchical} introduces an aggressive hierarchical approach which takes long-term choices first and commits them in a top-down fashion to reduces the amount of search required. Nonetheless, these works do not have a reactive high-level planning to adapts with a dynamic environment.

Recently, there is an interest to combine formal approaches to the robot manipulation problem. 
In the work \cite{dantam2016incremental}, it is presented an approach which combines Satisfiability Modulo Theories (SMT) in the discrete layer with sampling-based motion planning to achieve probabilistic completeness. It proposes a notation-independent task language to incorporate action effects. In \cite{he2015towards}, the task is specified in LTL formulas allowing richer expressiveness. These methods also do not consider changes in the object locations due to actions from the environment. Furthermore, the collision-free aspect of motion planning for moving obstacles is not considered. These obstacles appear in heterogeneous multi-robotic systems since other robots can be seen as moving obstacles to each other.

In this paper, we leverage a broad range of methods from literature to develop a reactive integrated mission and motion planning framework for multi-robot manipulation tasks. Our proposed framework has a hierarchical structure that accepts an LTL mission specification over the desired location of objects.
At the high-level, the mission planner considers uncontrollable robots in the workspace as an adversary players and synthesizes a  strategy that adapts to changes in the workspace enforced by uncontrollable robots. 
In the synergistic layer, according to a mission road-map, a robot manipulation plan is automatically synthesized. It encodes an integrated task and motion planning to a Bounded Satisfiability Checking (BSC) \cite{pradella2013bounded} specified by Bounded Prefix LTL formulas which guarantee dynamically feasible trajectories in the known dynamic workspace. BSC models consist of temporal logic with arithmetic terms rather than transition systems; thus, it allows integration of geometric details. During execution of the synthesized tasks, we employ safe motion primitive verified in Differential Dynamic Logic (d$\mathcal{L}$) \cite{platzer2010logical}. d$\mathcal{L}$ language can model non-deterministic hybrid systems to guarantee safety formally, considering delays,  of local motion planning algorithms such that the robot does not actively collide with obstacles observed by the sensors. 

Our main contribution lies in two folds: adapting reactive synthesis in the mission planner that enables robots to respond to change of workspace enforced by uncontrollable agents such human or other uncontrollable robots;  designing a synergistic layer that automatically constructs robot manipulation tasks ensuring dynamically-feasible and collision-free motion trajectories.

The remaining of this paper is organized as follows. The background and preliminaries are introduced in Section \ref{sec:preliminaries} which follows with the problem formulation in Section \ref{SEC:problem_formulation}. Section \ref{sec:rec_motion_plan}  presents  the reactive integrated task and motion planning layer. The reactive mission planning is addressed in Section \ref{sec:reactive_mission_synth}.  A warehouse manipulation example is presented in Section \ref{sec:simulation}.
The paper is concluded in Section \ref{SEC:CON}.

\section{Preliminaries}\label{sec:preliminaries}
In this section, we introduce the preliminary terminology and notations that are used throughout this paper. 
 

\subsection{$K$-Bounded Prefix LTL with Arithmetic Temporal Terms (LTL$_K$)}\label{sec:LTL}

We express the specification of an integrated task and motion planning based on Counter Linear Temporal Logic Over Constraint System CLTLB($\mathcal{D}$) defined in \cite{bersani2010bounded} and Bounded Linear Temporal Logic defined in \cite{latvala2004simple}. This language is interpreted over arithmetic and boolean terms $\mathcal{S} := \mathcal{V} \cup \mathcal{Q}$, where $\mathcal{V}$ is a set of continuous variables such that a variable $v \subseteq \mathcal{V}$ at instant $k$ is real-valued $v^{[k]} \in \mathbb{R}$, or integer-valued $v^{[k]} \in \mathbb{N}$; and $\mathcal{Q}$ is a set of boolean state variables such that a state variable $q \subseteq \mathcal{Q}$ at instant $k$ is boolean-valued $q^{[k]} \in \{\perp, \top\}$. Furthermore, arithmetic terms $v \in \mathcal{V}$ are valuated using temporal operators, this valuation is named arithmetic temporal term $\varphi$ such that the valuation function $\mathcal{I} : \varphi \times [1..K] \mapsto \mathbb{R} \times \mathbb{N}$. For short, we name LTL$_{K}$ for the $K$-Bounded Prefix LTL with Arithmetic Temporal Terms. 

\begin{definition}[Arithmetic Temporal Term]\label{def:att}
	A $K$-bounded prefix arithmetic temporal term $\varphi$ is defined as: $\varphi ::= s \mid \bigcirc \varphi$, 	where $s \subseteq S$ such that $s \in \mathbb{R}$ or $s \in \mathbb{N}$; and $\bigcirc$ stands for the next operator, i.e. $\mathcal{I}(\bigcirc r, k) = r^{[k+1]}$.
\end{definition}

The terms are interpreted using a labeling function $\mathcal{L} : \mathcal{D}_{\mathcal{V}} \times \mathcal{D}_{\mathcal{Q}} \mapsto 2^{\Pi}$, where $\Pi$ is a finite set of atomic propositions. Each atomic proposition $\pi \in \Pi$ is associated to a boolean term or to a linear arithmetic predicate over arithmetic temporal terms. An atomic proposition $\pi \in \Pi$ associated to a boolean term $\llbracket \pi \rrbracket := \{ q \in Q \}$ holds true at instant $k$ if and only if $q^{[k]}$ holds true. On the other hand, an atomic proposition $\pi \in \Pi$ associated to a linear arithmetic predicate over arithmetic temporal terms $\llbracket \pi \rrbracket := \{ \boldsymbol{h}^\intercal \boldsymbol{\varphi} \bowtie c \}$ holds at instant $k$ true if and only if $\boldsymbol{h}^\intercal \mathcal{I}(\boldsymbol{\varphi}, k) \bowtie c$ holds true, where $\bowtie$ is a relation operator ($\bowtie \,\in \{\leq \mid < \mid = \mid > \mid \geq\}$), $\boldsymbol{\varphi}$ is a vector of $n$ arithmetic temporal terms such that $\mathcal{I}(\boldsymbol{\varphi}) \in \mathbb{R}^{n_\mathbb{R}} \times \mathbb{N}^{n_\mathbb{N}}$ and $ n = n_\mathbb{R} + n_\mathbb{N}$, $\boldsymbol{h} \in \mathbb{R}^{n}$ and $c \in \mathbb{R}$. Therefore, a LTL$_K$ formula is defined as below.

\begin{definition}[Formula]\label{def:formula}
	A LTL$_K$ formula $\phi$ is defined as: $\phi ::= \pi \mid \neg \phi \mid \phi_1 \wedge \phi_2 \mid \bigcirc \phi  \mid \phi_1 \mathbf{U} \phi_2$, where $\pi \in \Pi$ is a atomic proposition, $\bigcirc$ stands for usual next operator, and $\mathbf{U}$ stands for usual until operators.
\end{definition}

Each formula defines a set of infinite words $\boldsymbol{\rho}$ over $2^{\Pi}$, i.e. $\boldsymbol{\rho} \in (2^{\Pi})^\omega$ is an infinite word. A $K$-bounded prefix is a sequence $\boldsymbol{\xi}^K = \boldsymbol{s}^{[0]}\boldsymbol{s}^{[1]}\boldsymbol{s}^{[2]}\dots\boldsymbol{s}^{[K]}$ and its word is $\boldsymbol{\rho}^K = \mathcal{L}(\boldsymbol{s}^{0})\mathcal{L}(\boldsymbol{s}^{[1]})\mathcal{L}(\boldsymbol{s}^{[2]})\dots\mathcal{L}(\boldsymbol{s}^{[K]})$. We consider a $K$-bounded prefix fragment of LTL language, meaning that we check if a formula $\phi$ defines a $K$-bounded prefix. When this formula defines a $K$-bounded prefix at instant $k$, we denote $\boldsymbol{\rho}^K \vDash_k \phi$. If the infinite path $\boldsymbol{\rho}$ is $K$-loop, then $\boldsymbol{\rho}^K \vDash_0 \phi$ will imply that $\phi$ defines the infinite word $\boldsymbol{\rho}$ \cite{biere1999symbolic}. Thus, we enforce that $\boldsymbol{\rho}$ is a $K$ loop by assuming a loop in the last state $\boldsymbol{s}^{[K]}$ to define the following semantics.

\begin{definition}[Semantics] \label{def:ltlK_semantics}
	The semantics of a LTL$_K$ formula $\phi$ at an instant $k \in [0..K]$ is as follow: 
    \begin{itemize}
    \item $\boldsymbol{\rho}^K \vDash_k \pi \Longleftrightarrow \pi \in 2^\Pi \wedge \pi \in \mathcal{L}(\boldsymbol{v}^{[k]})$, 
    \item $\boldsymbol{\rho}^K \vDash_k \neg \phi \Longleftrightarrow \boldsymbol{\rho} \nvDash_k^K \phi$, 
    \item $\boldsymbol{\rho}^K \vDash_k \phi_1 \wedge \phi_2 \Longleftrightarrow \boldsymbol{\rho}  \vDash_k^K \phi_1 \wedge \boldsymbol{\rho}  \vDash_k^K \phi_2$, 
    \item $\boldsymbol{\rho}^K \vDash_k \bigcirc \phi \Longleftrightarrow \boldsymbol{\rho}  \vDash_k^K \phi \vee k \notin [1..K-1]$, 
    \item $\boldsymbol{\rho}^K \vDash_k \phi_1 \mathbf{U} \phi_2 \Longleftrightarrow \begin{cases}
		\exists i \in [k..K]: \boldsymbol{\rho}  \vDash_i^K \phi_2 \wedge \\
		\forall j \in [k..i-1] \; \boldsymbol{\rho} \vDash_j^K \phi_1
		\end{cases}$.
    \end{itemize}
	Remark: $\mathcal{I}(\bigcirc r, K) = r^{[K]}$.
\end{definition}

Based on the grammar in Def.\ref{def:formula}, it can also use others common abbreviations, including: standard boolean, such as $\perp$, $\top$, $\vee$ and $\rightarrow$, $\Diamond \phi$ that stands for $\top \mathbf{U} \phi$, and it means that $\phi$ eventually holds before the last instant (included), $\square \phi$ that stands for $\neg \Diamond \neg \phi$, and it means that $\phi$ always holds until the last instant, $Last$ that stands for $\bigcirc \perp$ holds true only at last instant $K$. 

LTL$_K$ formula is used to specify properties over finite runs. In order to define properties over infinite words, we use LTL formulas. The semantic of an LTL formula is similar to LTL$_K$ but for infinite bound. The language of an LTL formula $\psi$  over variable $\mathcal{V}$ is defined as $\mathcal{L}(\psi)=\{ \omega \in {(\mathcal{D}_\mathcal{V}})^\omega | \omega \models \phi \}$. For further details on LTL model checking and it's semantic, the reader can refer to \cite{clarke_model_1999}.

\section{Problem Formulation} \label{SEC:problem_formulation}
To motivate this work, we provide a warehouse robotic example that will be used through this paper.  
\subsection{Motivating Example}\label{SEC:motivating}

We consider fully actuated ground mobile robots and, in this example, we used a Dublin's vehicle model,
\begin{equation}\label{eq:rdm1}
\dot{\boldsymbol{x}} = f(\boldsymbol{x},\boldsymbol{u}),  \qquad \boldsymbol{y} = H \boldsymbol{x}, 
\end{equation}
where $\boldsymbol{x} = (p^x, p^y, \theta, v, \omega) \in \mathcal{D}_{\mathcal{X}} \subseteq \mathbb{R}^5$ is the state, $\boldsymbol{u} = (a,\alpha) \in \mathcal{D}_{\mathcal{U}} \subseteq \mathbb{R}^2$ is the control input, $f(\boldsymbol{x},\boldsymbol{u}) = (v\cos \theta, v\sin \theta, \omega, a, \alpha)$, $H = diag(1,1,1) \in \mathbb{R}^{3,5}$, and $\boldsymbol{y} \in \mathcal{D}_{\mathcal{Y}} \subseteq  \mathbb{R}^3$ is the output. The states describe the robot positions $\boldsymbol{p} = (p^x,p^y)$, orientation $\theta$ and velocities $\boldsymbol{v} = (v,\omega)$, and the outputs describe the robot configuration $\boldsymbol{y} = (p^x,p^y,\theta)$. The translational acceleration $a$ is bounded $a \in [-B,A]$, and the translational $v$ and angular $\omega$ are linked by the nonholomonic constraint $-\dot{p}^x \sin \theta + \dot{p}^y \cos \theta = 0$.
We call $\boldsymbol{x}(t)$ as the state trajectory and $\boldsymbol{y}(t)$ as the output trajectory such that the robot dynamic model is under the initial condition $\boldsymbol{x}(0) = \boldsymbol{x}_{0}$.
Now consider a warehouse layout depicted in Figure \ref{FIG:map}, with having manipulator robots $R_1$, $R_2$, and $R_3$. We consider robot $R_1$ as our controllable system and other robots as part of the environment. The user does not know precisely the layout map,  although, the following abstracted scenario features are known to him: the robots initial positions $R_1H$ and $R_2H$; the scanning area $W_1$ and $W_2$, where the objects are sorted and placed by robot $R_2$; and the transportation locations $W_3$ and $W_4$ where objects are supposed to be dropped off.  Robot $R_1$ mission is to move the unloaded objects to the transportation places $W_3$ and $W_4$, and meanwhile avoid collision with other robots and any unexpected moving obstacles. This manipulation mission, at the high-level requires an strategy to adapt  with robot $R_2$ behavior, and at the low-level, needs motion planning based on the warehouse layout and other moving robots.

\begin{figure}[!ht] 
	\centering
	\includegraphics[width=70mm,height=35mm]{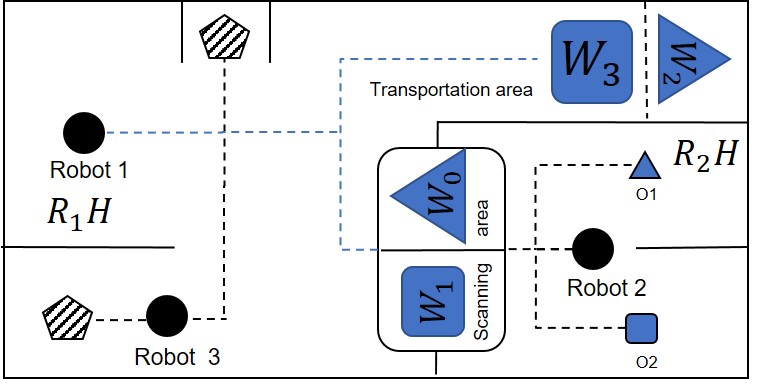}
	\caption{Warehouse layout and abstracted features.}
	
	
	\label{FIG:map}
\end{figure}

\subsection{Reactive Mission and Motion Planning}
We first, give a general overview of the integrated mission and motion planning framework depicted in Figure \ref{FIG:RITMP_framework}.  Consider  a group of manipulator  robots  $\mathcal{R} = \{R_1,\dots,R_{|R|}\}$ and a   workspace $\mathcal{W} \subset \mathbb{R}^2$ that contains a finite set of obstacles, and  movable objects. In the workspace there exist a set of locations of interest for the objects to be placed, and it is assumed each location is large enough to contain one and only one object. The mission requirement for the robots is to move the objects in the workspace to the desired locations.   We formalize this problem in a hierarchical structure with a reactive mission planner and an integrated task and motion planning layers.

The location of objects and robots positions constructs a  space $\mathcal{V}$ with a bounded domain $\mathcal{D}_\mathcal{V} $, that will be called task space in the rest of the paper.
 The objects and robots' location of interest  in the task space are abstracted   as discrete states in the mission planning layer. Let the set of all mission states be $\mathcal{M} = \{ \boldsymbol{m}_1 , \dots, \boldsymbol{m}_{|\mathcal{M}|} \} $, where $\boldsymbol{m}_i \in \mathcal{D}_\mathcal{V}$, and they are mutually exclusive $\bigcap_{i \in [1..|\mathcal{M}|]} \boldsymbol{m}_i = \emptyset$.  
This abstraction helps the user to be able to declaratively define the mission specification over location of objects. For instance, in the warehouse example, object $o_1$ eventually be at transportation area $W_2$.
The transition between these states is an integrated task and motion planning problem for a manipulator robot that can be designed and verified off-line.  We call these transitions, bounded-time tasks that are specified with LTL$_K$ formulas.   A bounded-time task can be seen as robot's skill on moving freely or carrying an object to a specific location in the workspace.  These skills are designed for a set of motion primitives that jointly satisfies the bounded-time tasks.   Motion primitives itself are reactive hybrid controllers that induce collision-free and dynamically-feasible motion trajectories in the described workspace, even when the robot meets an unexpected moving obstacles with maximum velocity $V_{\mathcal{B}^\prime}$. 

Considering $\mathcal{M}$ as a set of nodes and the transition between them as an edge set $\Gamma$, we can construct a  mission graph, capturing all feasible integrated task and motion planning for the given task space valuation, i.e $\mathcal{M}$. We consider a set of robots in the workspace as our controllable system and other robots as the system's  environment which their behaviors are uncontrollable but known to the system. 
The mission graph can be characterized as a turned-based two-player game arena with node set   $\mathcal{M}$ that is partitioned to system states $\mathcal{M}_s $, and environment mission states  $ \mathcal{M}_e$, meaning  the system takes actions at mission state  $\boldsymbol{m} \in \mathcal{M}_s$, and, otherwise, it is the environment robots turn.  
The edge set also is partitioned as $\Gamma = \Gamma_s \bigcup \Gamma_e$ representing the integrated task and motion planning specifications that can be realized by the system and its environment. With these ingredients, the reactive mission planning problem is defined as  synthesis of  system robots  winning strategy in the mission graph that satisfies the mission specification.

\begin{problem}[Reactive Mission and Motion Planning]\label{PROB:reactive_mimotion_planning}
	Given a high-level mission $\Psi$, a set of mission states $\mathcal{M}$, the behavior of the environment abstracted by $\mathcal{M}_e$ and $\Gamma_e$, a set of manipulator robots $\mathcal{R}$ with dynamic model (\ref{eq:rdm1}), an initial state for robot $\boldsymbol{x}_0$, a workspace $\mathcal{W}$, a set of static obstacles $\mathcal{CB}$, a maximum velocity for moving obstacles $V_{\mathcal{B}^\prime}$, synthesize a mission strategy which ensures $\Psi$ and yields trajectories that are collision-free and dynamical feasible in $\mathcal{W}$. 
\end{problem}

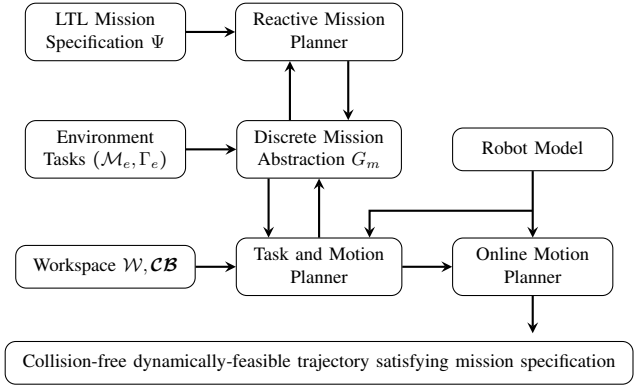
\begin{figure}[!ht]
	\begin{center}
		\begin{tikzpicture}[scale=.71,transform shape,node distance=1.5cm] `
		
		\node (MissionPlanner) [startstop,xshift=0cm] { \begin{tabular}{c}  Reactive Mission\\ Planner \end{tabular} };
		
		\node (AbsRef) [process, below of=MissionPlanner, xshift=0cm,yshift=-.7cm] {\begin{tabular}{c} Discrete Mission \\ Abstraction  $G_m $ \end{tabular} };
		
		\node (TaskMotion) [process, below of=AbsRef, xshift=0cm,yshift=-.7cm] {\begin{tabular}{c} Task and Motion \\ Planner \end{tabular} };	
		
		\node (MissionSpec) [process,  left of=MissionPlanner, xshift=-2.5cm,  yshift=.0cm] {\begin{tabular}{c} LTL Mission \\ Specification $\Psi$ \end{tabular}};
		
		\node (EnviromentTask) [process,  left of=AbsRef, xshift=-2.5cm,  yshift=-.0cm] {\begin{tabular}{c} Environment   \\ Tasks $(\mathcal{M}_e, \Gamma_e)$ \end{tabular}};
		
        \node (RobotModel) [process,  right of=AbsRef, xshift=2.5cm,  yshift=0.0cm] {\begin{tabular}{c} Robot   Model  \end{tabular}};
		
        
		\node (Workspace) [process,  left  of=TaskMotion, xshift=-2.50cm,  yshift=.0cm] {\begin{tabular}{c} Workspace $\mathcal{W},\boldsymbol{\mathcal{CB}}$  \end{tabular}};
		
		\node (OnlineMotionPlan) [process,  right of=TaskMotion, xshift=2.50cm,  yshift=.0cm] {\begin{tabular}{c} Online Motion \\ Planner   \end{tabular}};
		
		\node (Trajectory) [process,  below of=OnlineMotionPlan, xshift=-4cm,  yshift=-.3cm] {\begin{tabular}{c} Collision-free dynamically-feasible trajectory satisfying  mission specification\end{tabular}};	
        \coordinate[below of=RobotModel, xshift=.0cm,  yshift=0.35cm] (rmd);
		
		\draw [arrow] (MissionPlanner.315) -- (AbsRef.45);
		\draw [arrow] (AbsRef.135) -- (MissionPlanner.225);
		
		\draw [arrow] (TaskMotion) -- (AbsRef);
		\draw [arrow] (TaskMotion) -- (OnlineMotionPlan);
		\draw [arrow] (AbsRef.210) -- (TaskMotion.150);

		\draw [arrow] (MissionSpec) -- (MissionPlanner);
		\draw [arrow] (EnviromentTask) -- (AbsRef);
		
        \draw [arrow] (rmd) -| (TaskMotion.30); 
        \draw [arrow] (RobotModel) -- (OnlineMotionPlan); 
		\draw [arrow] (Workspace) -- (TaskMotion);
		\draw [arrow] (OnlineMotionPlan.south)  -- ++(.0,-.7)(Trajectory);
		
		\end{tikzpicture}
	\end{center}
	\caption{Integrated mission and motion planning framework.}
	\label{FIG:RITMP_framework}	
\end{figure}


\section{Reactive Task and Motion Planning} \label{sec:rec_motion_plan}

The transition between two mission states $\gamma_i \in \Gamma$ is implemented as a task and motion planning which must ensure collision-free, dynamically-feasible trajectories, and actions to manipulate objects. This task and motion planning considers the known workspace to synthesize a robust plan for a local motion planning which avoids collision with unexpected obstacles during runtime. Hence, we call it Reactive Integrated Task and Motion Planning (RITMP). 

The robots workspace $\mathcal{W}$ is a full-dimensional polytope defined in a two dimensional Euclidean space $\mathcal{W} \subset \mathbb{R}^2$ which specifies a fixed Cartesian frame where objects can move. A configuration space $\mathcal{Y}_i$ of a robot $\mathcal{R}_i$ is defined by $\boldsymbol{y}_i = (p_i^x, p_i^y,\theta_i)$, where $\boldsymbol{p} \in \mathbb{R}^2$ and $\theta \in (-\pi,\pi]$ are the robot position and orientation in the workspace, respectively. This workspace has a set of known static and rigid obstacles $\boldsymbol{B} = \{ \mathcal{B}_1,\dots,\mathcal{B}_{|\boldsymbol{B}|} \}$, which by considering a circular shaped robot, \textit{$\mathcal{C}$-obstacle} open convex polytopes regions $\mathcal{CB}_i \in \boldsymbol{\mathcal{CB}}$ can be precomputed such that $\mathcal{CB}_i = \{ \boldsymbol{y}_i \in \mathcal{Y}_i : \mathcal{R}_i(\boldsymbol{y}_i(t)) \cap \mathcal{B}_i = \emptyset \}$.
Besides moving, the robot can execute some tasks such as grasping and placing movable objects which changes the workspace. Task domains typically use a variety of notations defined over states and actions \cite{erdem2012answer,srivastava2014combined}. We model these actions into LTL$_K$ specifications using a notation-independent task domain.
\begin{definition}[Task Language]\label{def:tasklang}
	A task language is a set of strings of actions defined by the tuple $\mathcal{T} := \langle \mathcal{V}, \boldsymbol{A}, E, \boldsymbol{v}_0, \boldsymbol{v}_f \rangle$, where,
	\begin{itemize}
		\item $\mathcal{V}$ is a set of state variables $(v_1,\dots,v_{|\mathcal{V}|}) \in \mathcal{D}_\mathcal{V}$ which is the conjunction of robot configuration state space $(v_1,\dots,v_3) = (p^x,p^y,\theta) \in \mathcal{D}_\mathcal{Y}$ and non-motion state space $(v_3,\dots,v_{|\mathcal{V}|}) \in \mathcal{D}_{\mathcal{V}_t}$,
		\item $\boldsymbol{A}$ is a set of actions $a \in \boldsymbol{A}$,
		\item $E \subseteq (\mathcal{D}_{\mathcal{V}} \times \mathcal{D}_{\boldsymbol{A}} \times \mathcal{D}_{\mathcal{V}})$ is the set of \textit{symbolic transitions} $e \in E$ denoting $\text{pre}(a) \xrightarrow{a} \text{eff}(a)$, where $a \in \boldsymbol{A}$ is the operator, $\text{pre}(a), \text{eff}(a) \subseteq \mathcal{D}_{\mathcal{V}}$ are the precondition and effect set. 
        We specify $\text{pre}(a)$ and $\text{eff}(a)$ at instant $k$ by linear arithmetic predicates $\psi_{pre,a}(\boldsymbol{v}^{[k]})$ and $\psi_{eff,a}(\boldsymbol{v}^{[k]},\boldsymbol{v}^{[k+1]})$, respectively, such that $\boldsymbol{v}^{[k]} \in \text{pre}(a) \equiv \psi_{pre,a}(\boldsymbol{v}^{[k]})$ and $\boldsymbol{v}^{[k+1]} \in \text{eff}(a) \equiv \psi_{eff,a}(\boldsymbol{v}^{[k]},\boldsymbol{v}^{[k+1]})$,
		\item $\boldsymbol{v}_0, \boldsymbol{v}_f \in \mathcal{D}_{\mathcal{V}}$ is the initial and final condition of the task state variables.
	\end{itemize}
\end{definition}

We study the problem of synthesizing a hybrid control system which generates collision-free output trajectories of (\ref{eq:rdm1}) to guarantee task requirements $\mathcal{T}$ in a partially known workspace  with moving obstacles. 
\begin{definition}[Collision-free Dynamically Feasible Trajectory]\label{Def:coldyntraj}
	A robot output trajectory $\boldsymbol{y}_i(t)$ is called collision-free dynamically feasible trajectory if the following conditions hold.
	\begin{enumerate}		
    	\item It does not leave the workspace: $\boldsymbol{y}_i(t) \in \mathcal{W}$,
		\item It does not collide with known static obstacles $\boldsymbol{\mathcal{B}}$: $\mathcal{R}_i(\boldsymbol{y}_i(t)) \cap \boldsymbol{\mathcal{B}} = \emptyset$;
		\item It does not collide with any unexpected obstacle $\mathcal{B}^\prime$ moving with velocity up to $V_{\mathcal{B}^\prime}$ detected by sensors: $\mathcal{R}_i(\boldsymbol{y}_i(t)) \cap \mathcal{B}^\prime = \emptyset$,
        \item It is dynamically feasible: $\forall t \in \mathbb{R}_{\geq 0}$ there exist $\boldsymbol{u}(t) \in \mathcal{D}_{\mathcal{U}}$ such that $\boldsymbol{\dot{x}} = f(\boldsymbol(x),\boldsymbol{u})$ and $\boldsymbol{x}(0) = \boldsymbol{x}_0$.
	\end{enumerate}
\end{definition}

A valid task plan should generates a continuous output trajectory $\boldsymbol{y}_i(t)$ which is feasible in the dynamic model (\ref{eq:rdm1}) and does not collide with any obstacle or other agent in the workspace.
\begin{definition}[Valid Task Plan]\label{Def:validtaskplan}
	A task plan $\boldsymbol{T}_i = \{ a^{[1]},\dots,a^{[K]} \}$ for the robot $R_i \in \mathcal{R}$ is valid if the following conditions hold.
	\begin{enumerate}		
		\item Any robot trajectory $\boldsymbol{y}_i(t)$ generated with local motion planning following the plan $\boldsymbol{T}$ is a collision-free dynamically feasible trajectory,
		\item The plan $\boldsymbol{T}_i$ satisfies the task language $\mathcal{T}_i$, i.e. $\boldsymbol{v}^{[0]} = \boldsymbol{v}_0$, $\boldsymbol{v}^{[K]} = \boldsymbol{v}_f$, $\psi_{pre,a^{[k]}}(\boldsymbol{v}^{[k]})$ and $\psi_{eff,a^{[k]}}(\boldsymbol{v}^{[k]},\boldsymbol{v}^{[k+1]})$.
	\end{enumerate}
\end{definition}



The integrated task and motion planning is solved hierarchically with two layers: offline discrete planning and continuous planning.
First, the discrete planning finds an integrated task and motion plan $\boldsymbol{\xi}^K := \langle \boldsymbol{T}, \boldsymbol{Y}, \boldsymbol{P} \rangle$ for the known obstacles, including the objects placed in workspace, where $\boldsymbol{T} = \{ a^{[1]},\dots,a^{[K]} \}$ is a task plan, $\boldsymbol{Y}$ and $\boldsymbol{\mathcal{P}}$ forms a motion plan composed by an ordered set of target robot configuration $\boldsymbol{Y} := \{\boldsymbol{y}^{[1]},\dots,\boldsymbol{y}^{[K]} \}$ and a collision free tunnel in the workspace $\boldsymbol{\mathcal{P}} := \{\mathcal{P}^{[1]},\dots,\mathcal{P}^{[K]} \}$. A tunnel $\boldsymbol{\mathcal{P}}$ is a sequence of open and bounded convex polytopes $\mathcal{P}^{[k]} \in \boldsymbol{\mathcal{P}}$ which there exist at least one feasible trajectory $\boldsymbol{x}(t)$ to the target $\boldsymbol{y}^{[k]}$ generated by the action $a^{[k]}$. We encode the discrete problem with LTL$_K$ formulas that are solvable by an SMT solver such as Z3 \cite{de2008z3} using linear arithmetic constraints. Second, a continuous trajectory $\boldsymbol{x}(t)$ is generated for this tunnel using a local motion planning such that $\boldsymbol{y}(t)$ is passive safe for dynamic environment with moving obstacles, meaning that the robot stops before a collision happens. This safety property is verified offline and abstracted to the discrete planner using safety verification with Differential Dynamic Logic d$\mathcal{L}$ \cite{platzer2010logical}. In the following subsections we will present the discrete and continuous planning methods.


\subsection{Motion Primitives}\label{sec:conplan}

We aim to address the extreme computation of combining logic with dynamical constraints via compositional verification and synthesis. The Continuous Planner consists of a library of certified, reactive motion primitives in the form of a tunnel $\boldsymbol{\mathcal{P}}$ in the workspace and target configurations $\boldsymbol{Y}$ in the robot output space. These primitives are certified and reactive as it provides an online collision avoidance algorithm by taking into account sensor readings when synthesizing continuous trajectories. For a robot with dynamic model (\ref{eq:rdm1}), a passive safety property can be proved using Differential Dynamic Logic (d$\mathcal{L}$) \cite{mitsch2016formal}, which can be implemented with local motion planning such as Dynamic Window Approach (DWA) \cite{fox1997dynamic}. Furthermore, all free configuration space $\mathcal{D}_\mathcal{Y}$ is feasible for this robot, and a local optimization can find a feasible trajectory in a convex polytope.

\begin{lemma}
	There exist a collision-free trajectory between an initial $\boldsymbol{y}_0$ to a target configuration $\boldsymbol{y}^*$ inside a convex polytope $\mathcal{P}$ using Dynamic Window Approach (DWA) if the following conditions hold true:
	\begin{itemize}
		\item $\mathcal{P}$ is obstacle free, or there exist moving obstacles inside this polytope which will eventually move away from the robot path,		
		\item the initial state $\boldsymbol{x}_0$ is safe, meaning that the robot can stop before a collision,
		\item the initial configuration $\boldsymbol{y}_0$ is inside $\mathcal{P}$.
	\end{itemize}
\end{lemma}
\begin{proof}
	The system (\ref{eq:rdm1}) can be feedback linearized as a double integrator $\ddot{\tilde{\boldsymbol{x}}} = \tilde{\boldsymbol{u}}$ \cite{oriolo2002wmr}; thus, a trajectory to any position $(p^x,p^y)$ inside an obstacle free polytope $\mathcal{P}$ can be found by local optimization. Furthermore, if the robot is stopped, it can rotate to any orientation $\theta$. Therefore, if the initial state $\boldsymbol{x}_0$ is safe, and the initial $\boldsymbol{y}_0$ and the target configuration $\boldsymbol{y}^*$ are inside an obstacle free convex polytope $\mathcal{P}$, then the DWA can find a feasible trajectory.
\end{proof}

We ensure that a non-convex tunnel $\boldsymbol{\mathcal{P}}$ of obstacle free configuration space is feasible using DWA by requiring that this tunnel is composed by a sequence of convex polytopes $\boldsymbol{\mathcal{P}} := \{\mathcal{P}^{[1]},\dots,\mathcal{P}^{[K]} \}$ such that $\mathcal{P}^{[k]} \cap \mathcal{P}^{[k+1]} \neq \emptyset$. This condition can be specified in the following LTL$_K$ formula.
\begin{equation}\label{eq:safe}
\begin{aligned}
\phi_{safe}(\mathcal{W}, \boldsymbol{\mathcal{CB}}) \equiv \square \Big[ & \bigwedge_{i \in [1..n_F^w]} w_i \wedge \\
& \bigwedge_{i \in [1..|\mathcal{CB}|]} \bigvee_{j \in [1..n_{F}^{b,i}]} \big(b_{i,j} \wedge \bigcirc b_{i,j}\big) \Big],
\end{aligned}
\end{equation}
where $w_i \equiv \boldsymbol{h}_{w,i}^\intercal [p^x,p^y]^\intercal \leq c_{w,i}$, $b_{i,j} \equiv \boldsymbol{h}_{b,i,j}^\intercal [p^x,p^y]^\intercal \leq c_{b,i,j}$, $\boldsymbol{h}_{w,i} \in \mathbb{R}^2$ and $c_{w,i} \in \mathbb{R}$ are facet $i \in [1..n_F^w]$ parameters of the polytope $\mathcal{W}$, and $\boldsymbol{h}_{b,i,j} \in \mathbb{R}^2$ and $c_{b,i,j} \in \mathbb{R}$ are facet $j \in [1..n_F^{b,i}]$ parameters of the polytope $\mathcal{CB}_i$.

\begin{lemma}
	If a task and motion plan $\boldsymbol{\xi}^K$ satisfies $\phi_{safe}$, it will generate a collision free trajectory $\boldsymbol{y}(t)$ for a given workspace $\mathcal{W}$ with known static obstacles $\mathcal{B}$. 
\end{lemma}
\begin{proof}
	A polytope $\mathcal{P}$ is obstacle free inside the workspace $\mathcal{W}$ if it is inside of at least one of the half-plane of each $\mathcal{C}$-obstacle region $\mathcal{CB}_i \in \mathcal{CB}$. Furthermore, two sequent polytopes $\mathcal{P}^{[l]}$ and $\mathcal{P}^{[l+1]}$ intersects with each other ($\mathcal{P}^{[l]} \cap \mathcal{P}^{[l+1]} \neq \emptyset$) if and only if there exist a point that is inside both polytopes. Therefore, a obstacle free polytope is formed by one of active hyper-plane (a hyper plan which $b_{i,j}$ holds true)	for each obstacle $\mathcal{CB}_i$ and the workspace hyper-planes. Moreover, sequent polytopes has non-empty intersection because for each active obstacle hyper-plan is active for next robot position, i.e. $\big(b_{i,j} \wedge \bigcirc b_{i,j}\big)$. 
\end{proof}
Local motion planning, such as DWA, not only efficiently computes trajectories, but  also is robust to unknown obstacles. However, dynamic obstacles can turn these trajectories unsafe before the robot can react due delays such as sensors sampling time and computation. Hence, we design a minimally invasive safety supervisor which ensure that the robot stops before the collision occurs. This supervisor gets data from radar sensor to detect when the collision is imminent to take over and stop the robot. A similar strategy is used in other collision avoidance controllers in the literature \cite{borrmann2015control}. 

\begin{definition}[Safe Motion Primitive]
	A safe motion primitive is hybrid control system with two modes of operation: inactive and override. The dynamic model is defined in (\ref{eq:rdm1}) for a circular shaped robot with radius $D_s$. The moving obstacle $\mathcal{B}^\prime$ dynamic model is unknown, but its maximum velocity $\|\boldsymbol{v}_{\mathcal{B}^\prime}\| \leq V_{\mathcal{B}^\prime}$ is known, and its closest point $\boldsymbol{p}_{\mathcal{B}^\prime}^* \in \mathcal{B}^\prime$ to the robot is detected by a radar sensor. In the drive mode, an over-the-shelf local motion planning generates the robot trajectory when $safe \equiv \|\boldsymbol{p}-\boldsymbol{p}_{\mathcal{B}^\prime}^*\|_{\infty} > \frac{v^2}{2B} + V_{\mathcal{B}^\prime}\big(\epsilon + \frac{v+A\epsilon}{B}\big) + \big(\frac{A}{B}+1\big)\big(\frac{A}{2}\epsilon^2 + \epsilon v\big) + D_s$ holds true, where $\epsilon$ is the maximum reaction delay. Otherwise, the mode override stops the robot ($a = -A$).
\end{definition}

This supervisor is synthesized offline and abstracted to the high-level discrete planner using Differential Dynamic Logic d$\mathcal{L}$ \cite{platzer2010logical} verification. 

\begin{lemma}
	Any state trajectory $\boldsymbol{x}(t)$ of a Safe Motion Primitive stops before the collision, i.e. $\phi_{pf} \equiv (v=0) \vee \big(\|\boldsymbol{p} - \boldsymbol{p}_{\mathcal{B}^\prime}^*\| > \frac{v^2}{2B}+V_{\mathcal{B}^\prime}\frac{v}{B}+D_s\big)$, if $\phi_{pf}$ holds true initially. 
\end{lemma}
\begin{proof}
	In \cite{mitsch2016formal}, it was proved with KeYmaera that the translational acceleration $a$ ensures the invariance of $\phi_{pf}$ if $a \in [-B,A]$ if $safe$ holds true, otherwise $a = -B$.
	Moreover, $(v=0) \vee \big(\|\boldsymbol{p} - \boldsymbol{p}_{\mathcal{B}^\prime}^*\| > \frac{v^2}{2B}+V_{\mathcal{B}^\prime}\frac{v}{B}+D_s\big) \rightarrow (v=0) \vee \forall \boldsymbol{p}_{\mathcal{B}^\prime} \in \mathcal{B}^\prime \, \big(\|\boldsymbol{p} - \boldsymbol{p}_{\mathcal{B}^\prime}\| > D_s\big)$; thus, the robot will never actively collide with a moving obstacle.
\end{proof}

\subsection{High-Level Discrete Planner}\label{sec:dplan}

The High-Level Discrete Planner shown in Algorithm \ref{alg:simple} synthesizes a sequence of motion primitives $\boldsymbol{\xi}^K := \langle \boldsymbol{T}, \boldsymbol{Y}, \boldsymbol{\mathcal{P}} \rangle$, which each action $a^{[k]} \in \boldsymbol{T}$ is parametrized by a target configuration $\boldsymbol{y}^{[k]} \in \mathcal{Y}$ and a feasible convex polytope $\mathcal{P}^{[k]} \subseteq \mathcal{W}$ in the workspace. A feasible output space satisfies the LTL$_K$ formula $\phi_{safe}$. Moreover, besides moving, the robot can also realize some tasks that change the environment which are specified in the task language $\mathcal{T}$.

		
		
		

\begin{algorithm}
	\caption{Discrete Planner Algorithm} 
		\label{alg:simple}
	\begin{algorithmic}[1]
		\Statex
		\Function{ITMP}{$\phi , K_{max}$}
		\Let{$K$}{$1$} \Comment{Initialize the horizon.}
		\Let{Satus}{UNSAT} 
		\While{Status = UNSAT $\wedge K \leq K_{max}$}
		
		\Let{(Status, $\boldsymbol{\xi}^K$)}{Check($K,\phi$)}
		\Let{$K$} {$K+1$}
		
		\EndWhile
		
		\State \Return{Status, $\boldsymbol{\xi}^K$}
		\EndFunction
	\end{algorithmic}
\end{algorithm}

The task language is encoded in LTL$_K$ formula $\phi_{\mathcal{T}}$ over state variables $V \cup a$, where the variable $a \in \mathbb{Z}$ is a state variable defining the motion primitive taken at each time instant in the discrete plan. This formula states that the primitive must be always valid ($1 \leq a \leq |A|$) and every time that a primitive is selected, its preconditions and effects must hold true ($\psi_{pre,a_i}(\boldsymbol{v}) \wedge \psi_{eff,a_i}(\boldsymbol{v}, \bigcirc \boldsymbol{v})$). Since at last time instant in the discrete plan we do not select any primitive, we ignore their constraints at this instant, i.e. $\neg Last$.
\begin{equation}
\begin{aligned}
& \begin{aligned} 
\phi_{\mathcal{T}}(\boldsymbol{v}_0,\boldsymbol{v}_f) \equiv & (\boldsymbol{v} = \boldsymbol{v}_0) \wedge \square (Last \rightarrow \boldsymbol{v} = \boldsymbol{v}_f)  \\
& \wedge \square \big[1 \leq a \leq |A|   \wedge \bigwedge_{i \in [1..|A|]} \phi_{\boldsymbol{A}}^{[i]}\big] 
\end{aligned}\\
& \phi_{\boldsymbol{A}}^{[i]} \equiv \big((a = i) \wedge \neg Last\big) \rightarrow \big(\psi_{pre,a_i}(\boldsymbol{v}) \wedge \psi_{eff,a_i}(\boldsymbol{v}, \bigcirc \boldsymbol{v})\big).
\end{aligned}
\end{equation}

Therefore, the task language is encoded in the formula $\phi_{\mathcal{T}}$, and the workspace $\mathcal{W}$ and known static obstacles $\mathcal{B}$ to the formula $\phi_{safe}$.

\begin{thm}	\label{thm:task_spec_synth}
	For a circular shaped robot with dynamic model (\ref{eq:rdm1}), given an initial state $\boldsymbol{x}_0$, a task language $\mathcal{T}$, a workspace $\mathcal{W}$, a set of static obstacles $\mathcal{B}$ and a maximum velocity for moving obstacles $V_{\mathcal{B}^\prime}$, a task and motion plan $\boldsymbol{\xi}^L$ which satisfies $\phi_{safe} \wedge \phi_{\mathcal{T}}$ generates a valid task plan $\boldsymbol{T}_i$. 
\end{thm}
\begin{proof}
$\phi_{safe}$ ensures that a collision free trajectory exists for a robot with dynamic model (\ref{eq:rdm1}) in the workspace $\mathcal{W}$ with static obstacles $\mathcal{B}$, i.e. condition 1, 2 and 4 in Def. \ref{Def:coldyntraj}. Moreover, all primitives implement the minimally invasive safety supervisor; thus, they are passive safe for moving obstacles with a known maximum velocity $V_{\mathcal{B}^\prime}$, i.e. condition 3 and 4 in Def. \ref{Def:coldyntraj}. Finally, $\phi_{\boldsymbol{T}}$ ensures the action-change behaviors, i.e. condition 2 in Def. \ref{Def:validtaskplan}. Therefore, any discrete task and motion plan which satisfies $\phi_{safe} \wedge \phi_{\mathcal{T}}$ will generate a valid task plan.
\end{proof}

\section{Reactive Mission Strategy Synthesis} \label{sec:reactive_mission_synth}
The solution of task and motion planning problem provides a set of safe motion primitives that jointly satisfy a bounded-time task specification.  The mission specification, however,  can be defined over more complex scenario and possibly requires an infinite sequence of bounded-time tasks to be accomplished.  
We consider the  controllable robot team, $R_s \subseteq \mathcal{R}$,   and the environment  robots as $R_e:=\mathcal{R} \backslash R_s$.
The mission planning problem can be characterized to a reactive synthesis formalism by constructing a two-player game arena, denoted as mission graph $G_m$ with having   $R_e$ and $R_s$ as the players.  
A two-player symbolic turned-based game arena is a directed graph which models a reactive system interacting with its environment \cite{moarref2016compositional}.  It  defines over variable set $\mathcal{V}$ 
 that at each turn either system or environment takes an action and updates the vtask space $\mathcal{V}$. 
 The mission graph is   a tuple $G_m=(\mathcal{V},\Gamma,\Delta)$, where $\Gamma$ is finite  set of actions, and $\Delta$   is predicate over $ \mathcal{V} \cup \Gamma \cup  \mathcal{V}'  $ defining the transition relation in $G_m$.  Let's assume there exist an auxiliary variable with domain $\sigma \in \mathcal{V},   \mathcal{D}_\sigma=\{1,2\}$  defines which player's turn is to take an action.  
 Let's denote player-1 as  environment robots $R_e$ with a set of state  $M_e=\{m \in \mathcal{D}_\mathcal{V} |  m_{|\sigma}=1  \}$, and player-2 as system robots $R_s$ with set of states  $M_s=\{m \in \mathcal{D}_\mathcal{V} |  m_{|\sigma}=2  \}$.
Let's also define state set $\mathcal{M}=\mathcal{M}_e \cup \mathcal{M}_s$,
A $run$   $\omega = m_0m_1\dots$ is a sequence of states  $m_i \in \mathcal{M}$   where all pairs $(s_i,a,s^{\prime }_{i+1}) \models \Delta$ for  $i \geq 0$ and $a \in \Gamma$.

A  transition between mission states  $(m,m') \in \mathcal{M}$   is an integrated task and motion problem defined over a task language $\mathcal{T} = \langle \mathcal{V}, \boldsymbol{A}, E, m, m' \rangle$.  Here the aim is to construct an LTL$_K$ task specification $\phi$, over the task language  such that  there exist a valid task and motion plan   that satisfies $\phi$ with  initial task space  $m$  and final task space $m'$. 
Given mission states and the all the corresponding task languages, the mission planning problem is to construct the mission graph,   and then find a  strategy for $R_s$ player,  that satisfies the mission specification. 
\begin{definition} [Strategy]\label{defn:strategy}
 A strategy for player-$\sigma$ is a partial function $\mathcal{S}^\sigma: {\mathcal{D}_{\mathcal{V}}}^*\cdot  \mathcal{D}_{\mathcal{V}}^\sigma \to \Gamma$ such  that  for all prefix of runs ending in player-$\sigma$'s state, $\{r.m^\sigma \in {\mathcal{D}_\mathcal{V}}^*| m^\sigma \in \mathcal{D}^\sigma_\mathcal{V} \}$, if $\mathcal{A}(m^\sigma) \neq \emptyset$, choose a successor state from $\mathcal{A}(m^\sigma)$, that $(m^\sigma,\mathcal{S}^\sigma(r,s), m') \in \Delta$ holds, where $m'$ is a prime of  $m^\sigma$. 
 \end{definition}
  A mission game  with winning condition is  $(G_m,\Psi_{init},\Psi)$, where $\Psi_{init}$ is a predicate over $\mathcal{V}$ specifies  initial state of the game, and $\Psi$ is LTL formula specifies mission specification for  $R_s$.  
Consider $\mathcal{S}^\sigma$ be a strategy for player-$\sigma$, and $m \in \mathcal{M}$, the $outcomes(\mathcal{S}^1,\mathcal{S}^2,m)$ is all the runs in the form of  $m_0 m_1 \cdots$ that $m_0 = m$, and $(m_i,\mathcal{S}^\sigma(m_0 \cdots m_i),m_{i+1}) \in \Delta$. 
A strategy $\mathcal{S}^2$ is winning if for all strategies of player$-1$, $\mathcal{S}^1$, and all the states that satisfies the initial conditions, $m \models \Psi_{init}$, the $outcomes(\mathcal{S}^1,\mathcal{S}^2,m) \subseteq \mathcal{L}(\Psi)$. An LTL specification $\Psi$ is called realizable in $G_m$ with initial condition $\psi_{init}$ if and only if there exist a winning strategy $\mathcal{S}^2$. 
We call the winning strategy $\mathcal{S^2}$ mission strategy. 
The mission strategy synthesis is designed in two steps. 
First, we construct a   mission graph that captures all the feasible transitions over mission states for a known workspace. 
The next step is to find a winning strategy over mission graph that enforces the mission specification.  

\subsection{Mission Graph}
Mission graph is a symbolic  turned-based  two-player game arena $G_m=(\mathcal{V},\Gamma,\Delta)$, defined over task space $\mathcal{V}$.
The  variable states set is  $\mathcal{M}$ with  partitions of  environment state $\mathcal{M}_e=\{m^e_1, \dots, m^e_{|\mathcal{M}_e|} \}$ and system states $\mathcal{M}_s=\{m^s_1, \dots, m^s_{|\mathcal{M}_s|} \}$.  Let's define variable $\alpha \in\{s,e\}$ and $\bar \alpha= \{e,s\} \backslash \alpha$ for simplicity on notations.  
Let $succ(m)  \subseteq \mathcal{M}$ be set of all mission state that  a transition from state $m \in \mathcal{M}$ is feasible.   
The action set  $\Gamma = \Gamma_e \bigcup \Gamma_s$, and the transition relation $\Delta = \Delta_e \bigcup \Delta_s$  can be partitioned to   the environment and system players, where $\Delta_\alpha= \mathcal{V} \bigcup \Gamma_\alpha \bigcup \mathcal{V}'$.
We assume the tuple  $G_e=(\mathcal{V},\Gamma_e,\Delta_e)$  is pre-designed and known to the mission planner.  
The system's action set, $\Gamma_s=\bigcup \gamma_i$, is designed over a  bounded-time task specification $\phi_i$ that at $m^s_i \in \mathcal{M}_s$, can enforce  task space to $m^e_j \in  succ(m^s_i)$. If there exist such $\phi_i$, we use a symbolic representation of it as an action, denoted by a function $symb: \phi \to \gamma$ , and   add transition tuple $(m^s_i,\gamma_i,m^e_j) $ to $\Delta_s$. This procedure is presented in Algorithm \ref{alg:mission_graph}. 


\begin{algorithm} [!h]
	\caption{Mission Graph Synthesis Algorithm} 
		\label{alg:mission_graph}
	\begin{algorithmic}[1]
		\Require{$\mathcal{M}_s,G_e,\mathcal{W},\boldsymbol{\mathcal{CB}},K_{max}$}
		\Ensure{ $G_m=(\mathcal{V},\Gamma,\Delta)$}
		\Let{$i,j$}{$1$} \Comment{Initialize the counter.}
		\While{$i  < |\mathcal{M}_s|$}
           \While{$j \le |\mathcal{M}_e|$}
				 \Let {$\phi_k$} {$\phi_{safe}(\mathcal{W},\boldsymbol{\mathcal{CB}}) \wedge \phi_{\mathcal{T}}(m^s_i,m^e_j)$}
				\Let {$(Status, \boldsymbol{\xi}^K)$} {$ITMP(\phi_k, K_{max})$ }
              \If {$Status =$ SAT} 
                	\State $\gamma_i:=Symbl(\phi_k)$  \Comment{symbolic action for $\phi_i$.}
               		\State $\Gamma_s \gets \Gamma_s \cup \gamma_i$
                	\State  $\Delta \gets \Delta \cup (m_i,\gamma_i,m_j)$
             \EndIf
			\Let{$j$} {$j+1$}
			\EndWhile
            \Let{$i$} {$i+1$}
		\EndWhile
		
		\State \Return{ $G_m=(\mathcal{V},\Gamma_s \bigcup \Gamma_e,\Delta \bigcup 
        \Delta_e )$}
	\end{algorithmic}
\end{algorithm}
The synthesized $G_m$ has a deterministic system's transition function $\Delta_S$,  since for any $m^s \in \mathcal{M}_s$ and  all $m^e \in succ(m_k^s)$,  the corresponding  task specification, $\phi_k:=\phi_{safe} \wedge \phi_\mathcal{T}(m^s,m^e)$, if exist,  is unique. The environment transition function, $\Delta_e$, however can be non-deterministic.

\subsection{Mission Strategy Synthesis}
Construction of $G_m$  is a bottom-up approach that provides an abstracted model of all robots behavior for the mission planning layer.   The mission planning, on the other hand, is top-down design over mission graph that synthesis a finite mission strategy model that satisfies a given mission specification. 
Synthesis of mission strategy can be converted to a standard problem of solving symbolic two-players game \cite{moarref2016compositional}.   

\begin{prop} [Existence of Mission Strategy \cite{moarref2016compositional}] \label{prop:existence_mission_strategy}
Given initial configuration of task space  described  by predicate $\Psi_{init}$ over task space $\mathcal{V}$, and let mission specification $\Psi$ be an LTL formula defined  over mission state $\mathcal{M}$,   mission strategy $\mathcal{S}$ exist if   the system player  has a winning strategy in  $(G_m,\Psi_{init},\Psi)$ . 
\end{prop}

If there exist a mission strategy,  a finite-state transducer can be synthesized  that accepts all the runs induced by $\mathcal{S}$  \cite{vardi1996automata}. Let $G^m_\mathcal{S}$ be mission strategy  transducer
that satisfies the mission specification,   $\mathcal{L}(G^m_\mathcal{S}) \subseteq \mathcal{L}(\Psi)$, it is important that the induced motion trajectory from $G^m_\mathcal{S}$
be feasible in the workspace.  

\begin{thm} [Correctness of  Mission Strategy] \label{thm:correct_mission_strategy}
$\mathcal{L}(G^m_\mathcal{S}) $  induces  collision-free and dynamically-feasible trajectories.  
\end{thm}

\begin{proof}
We proof by construction that all the runs over mission states that induced by strategy model $G^m_\mathcal{S}$ generates  collision-free and dynamically-feasible trajectories as it is defined in Def \ref{Def:coldyntraj}. 
Let a run be $\omega \in \mathcal{L}(G^m_\mathcal{S})$, mission states   $m=\omega_i$ and $m'=\omega_{i+1} $, for $i \ge 0$ and $\{m,m'\} \in \mathcal{M} $. We can define   a task language $\mathcal{T}=\langle \mathcal{V}, \boldsymbol{A}, E, m, m' \rangle$ based on  Def. \ref{Def:validtaskplan} and construct an LTL$_K$ task specification by using    Theorem  \ref{thm:task_spec_synth},  as  $\phi=\phi_{safe} \wedge \phi_\mathcal{T}(m,m')$. Since  $m' \in succ(m)$ and  by definition of  $G_m$ given in Algorithm \ref{alg:mission_graph}, we know there exist a valid task plan that it's integrated task and motion plan, $\boldsymbol{\xi}^L$,  satisfies $\phi$. Hence, by Def. \ref{Def:validtaskplan} the induced trajectories are collision-free and dynamically-feasible.
\end{proof}

	
	
	

\section{Simulation} \label{sec:simulation}

We present results of our approach in simulation to evaluate an automatic synthesis of controllers for a multi-robotic system in a dynamic environment with a moving obstacle. The synthesis of
task and motion planning  is encoded to an SMT problem and solved with Z3 \cite{de2008z3} for Pioneer P3-DX Robots \footnote{\url{http://www.mobilerobots.com/ResearchRobots/PioneerP3DX.aspx}} equipped with gripper 
and a laser sensor. We simulate the execution of the generated controllers with MobileSim \footnote{\url{http://robots.mobilerobots.com/MobileSim/download/current/README.html}}.


Our evaluation is illustrated in a warehouse scenario presented in Sec. \ref{SEC:motivating} where a set of location of objects $L = \{ W_0, \dots, W_4 \}$ is given as in Fig. \ref{fig:sim1}. A task and motion plan is generated offline for all $(m_i^s,\lambda_i,m_j^e) \in \Delta_s$. For example, Fig. \ref{fig:sim2}, \ref{fig:sim3} and \ref{fig:sim4} shows a plan for moving an object in the scanning area $W_1$ to the transportation area $W_4$. The blue line is a continuous output trajectory $\boldsymbol{y}(t)$ for the robot $R_1$ which is generated at runtime considering the laser sensors such that ensures a task and motion plan synthesized offline and no collision with a priori unknown obstacles moving up to $0.8 m/s$. This robot crosses with the robot $R_3$ trajectory, the green line in Fig. \ref{fig:sim3} and \ref{fig:sim4}, while moving to drop the object off and returning home. At this moment, this robot automatically finds a collision-safe detour. Hence, the designed controller is reactive not only in the mission layer but also in the motion layer. Consequently, this controller is robust to unexpected changes in the environment. 

\begin{figure}
	\begin{subfigure}{.24\textwidth}
		\centering
		\includegraphics[width=\textwidth]{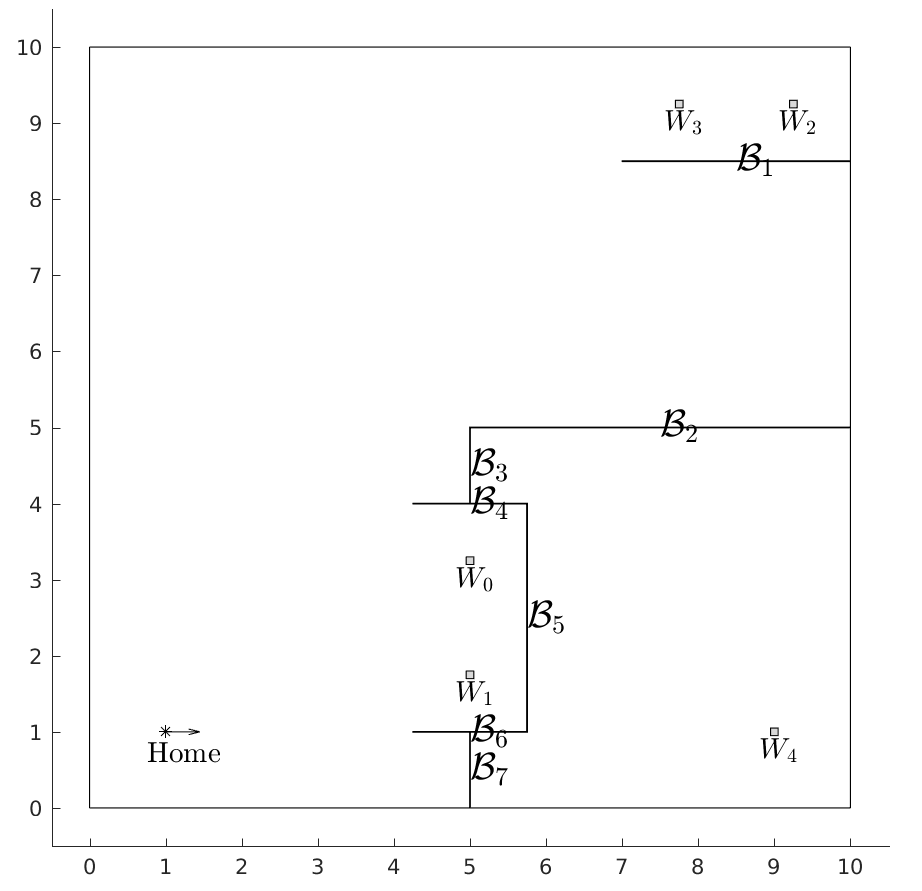}
		\caption{Pre-defined locations.}
		\label{fig:sim1}
	\end{subfigure}
		\begin{subfigure}{.24\textwidth}
			\centering
			\includegraphics[width=\textwidth]{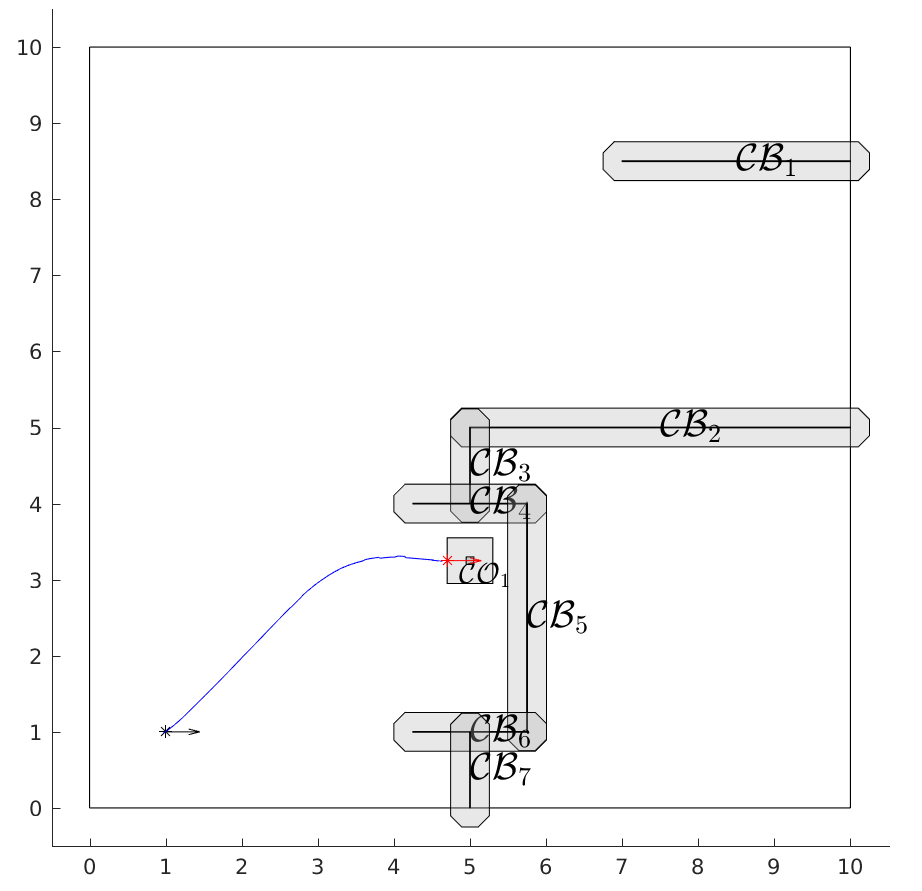}
			\caption{Pick Up}
			\label{fig:sim2}
		\end{subfigure}\\
	\begin{subfigure}{.24\textwidth}
		\centering
		\includegraphics[width=\textwidth]{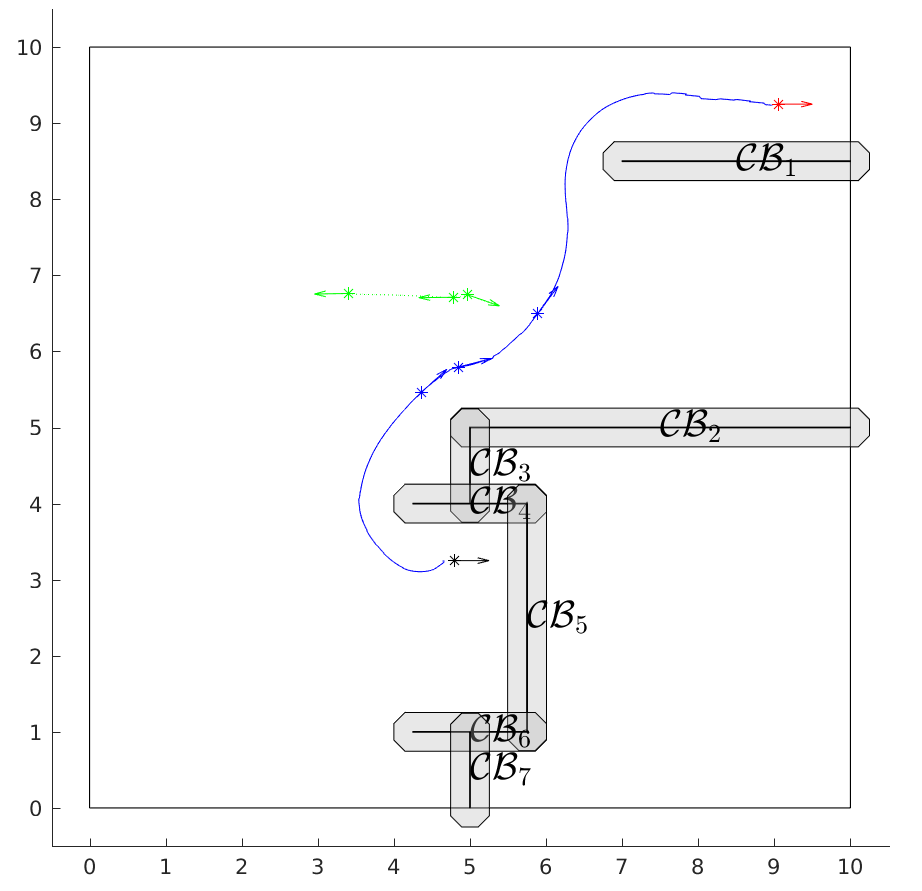}
		\caption{Drop Off}
		\label{fig:sim3}
	\end{subfigure}
		\begin{subfigure}{.24\textwidth}
			\centering
			\includegraphics[width=\textwidth]{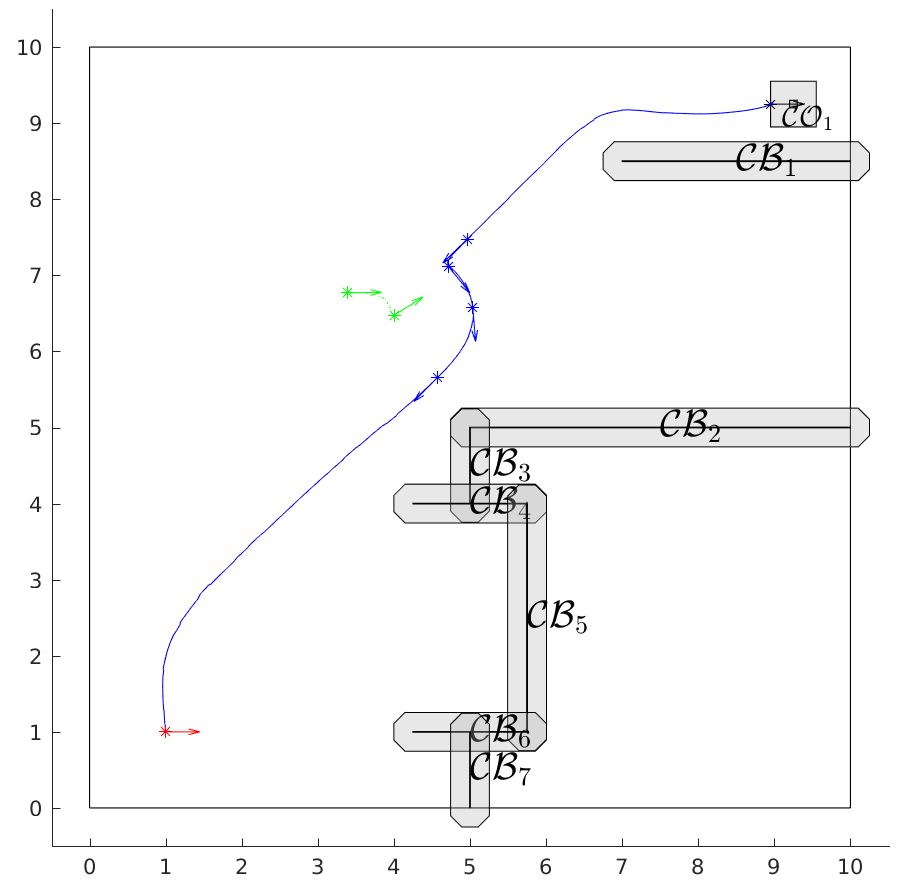}
			\caption{Go Home}
			\label{fig:sim4}
		\end{subfigure}\\
    \caption{An execution of a task and motion plan to fetch an object at $W_0$ to $W_2$.}
    \label{fig:sim1-4}
\end{figure}

\section{Conclusion}\label{SEC:CON}
In this paper, we addressed integrated mission and motion planning problem for robotic manipulator systems operating in a partially known environment. 
The main advantages of the proposed approach are in two-folds:  a) The high-level mission specification can be defined declaratively on desired robots and objects locations. The framework synthesis the required robot task to move the objects according to the mission states; b) The framework offers reactivity in both motion planning layer to handle  unknown moving obstacles, and in the mission planning to adjust the robot's strategy with respect to change of environment enforced by others robots operating in the workspace.


\bibliographystyle{IEEEtran}        
\bibliography{ACC_2018V2}

\end{document}